\documentclass[letterpaper]{article}
\usepackage{aaai24}  
\usepackage{times}  
\usepackage{helvet}  
\usepackage{courier}  
\usepackage[hyphens]{url}  
\usepackage{graphicx} 
\usepackage{natbib}
\usepackage{caption}
\newcommand{\shortver}[1]{#1} % short version; comment out for long version
\usepackage{nicefrac}  % compact symbols for 1/2, etc.
\usepackage{dsfont}    % double stroke font
\usepackage{amsfonts}  % blackboard math symbols
\usepackage{amsmath}   
\usepackage{amsthm}
\usepackage{thmtools}
\usepackage{appendix}            % styling for appendix
\usepackage[utf8]{inputenc}      % allow utf-8 input
\usepackage[T1]{fontenc}         % use 8-bit T1 fonts
\usepackage{booktabs}            % nice tables
\usepackage{microtype}           % microtypography
\usepackage{xcolor}              % colors
\usepackage{url}                 % simple URL typesetting
\usepackage{outlines} % shorthand for nested itemize
\usepackage{cleveref}            % easier references
\crefname{equation}{equation}{Equations}
\creflabelformat{equation}{#2\textup{#1}#3}

\newcount\Comments
\newcommand{\disablecomments}{\Comments=0} % disable comments
\definecolor{darkgreen}{rgb}{0,0.5,0}
\definecolor{darkred}{rgb}{0.7,0,0}
\definecolor{purple}{rgb}{0.5,0,0.5}
\definecolor{blue}{rgb}{0,0,0.5}
\definecolor{orange}{rgb}{0.8, 0.3,0}
\newcommand{\kibitz}[2]{\ifnum\Comments=1\textcolor{#1}{#2}\fi}

\newcount\Anonymous
\newcommand{\hideifanonymous}[1]{\ifnum\Anonymous=0#1\fi}
\newcommand{\showifanonymous}[1]{\ifnum\Anonymous=1#1\fi}

\usepackage{ amssymb }

\newcommand{\R}{\mathds{R}}
\newcommand{\E}{\mathop{\mathds{E}}}

\newcommand{\Datasets}{\mathcal{D}}

\newcommand{\Obj}{L}
\newcommand{\Agg}{\mathcal{L}}
\newcommand{\ErrorRate}{\Obj_{\textrm{Err}}}
 
\newcommand{\MAE}{MAE}
\newcommand{\MAEloss}{\Obj_{\textrm{MAE}}}
\newcommand{\NLL}{\Obj_{\textrm{NLL}}} 
\newcommand{\CrossEntropy}{\Obj_{\textrm{CE}}} 
\newcommand{\KL}{\Obj_{\textrm{KL}}} 
\newcommand{\Brier}{\Obj_{\textrm{Brier}}} 
\newcommand{\SquaredLTwo}{\Obj_{\textrm{L2}}}

\usepackage{multibib}
\newcites{Appx}{References}

\usepackage{thm-restate}
\declaretheoremstyle[
    style=definition, 
    bodyfont=\normalfont\itshape,
    numberwithin=section
]{mystyle}
\declaretheoremstyle[
    headfont={\it}, 
    bodyfont=\normalfont,
]{myproofstyle}
\declaretheoremstyle[
    style=definition, 
    bodyfont=\normalfont\itshape,
]{myaxiomstyle}

\declaretheorem[style=mystyle]{theorem}
\declaretheorem[style=mystyle, sibling=theorem]{definition, example, lemma, proposition, remark, corollary, conjecture}
\declaretheorem[style=myaxiomstyle, unnumbered]{axiom}

\declaretheorem[name=Proof, style=myproofstyle, qed=$\qedsymbol$, unnumbered]{prf}

\declaretheoremstyle[
    style=definition, 
    bodyfont=\normalfont\itshape,
]{mystyle_nosectionnumbers}
\declaretheorem[name=Definition, style=mystyle_nosectionnumbers, unnumbered]{definition*}

\usepackage{enumitem}
\setlist[itemize]{leftmargin=1em, topsep=0pt, partopsep=0pt, parsep=0pt}

\usepackage{graphicx}
\usepackage{float}
\usepackage{multirow}
\usepackage{adjustbox}
\usepackage{array}
\newcolumntype{R}[2]{%
    >{\adjustbox{angle=#1,lap=\width-(#2)}\bgroup}%
    l%
    <{\egroup}%
}
\newcommand*\rot{\multicolumn{1}{R{35}{1em}}}% no optional argument here, please!

\usepackage{pifont}
\newcommand{\yes }{{\bf \ding{51}}}
\newcommand{\no  }{--}

\providecommand{\shortver}[1]{}
\shortver{\newcommand{\longver}[1]{}}
\providecommand{\longver}[1]{#1}

\shortver{
    \newenvironment{kequation}{$}{$}
    \newenvironment{kequation*}{$}{$ }
    \newcommand{\minipar}[1]{\medskip\noindent\textbf{{#1.}}~}
    \newcommand{\superminipar}[1]{\smallskip\noindent\textbf{\emph{#1.}~~}}
    \newcommand{\micropar}[1]{\noindent\emph{#1.}~}
    \let\subsection\minipar
    \let\subsubsection\superminipar
    \let\paragraph\micropar
}
\longver{
    
    \newenvironment{kequation*}{\begin{equation*}}{\end{equation*}}
}
\disablecomments             % comment to enable comments; uncomment to disable comments
\title{How to Evaluate Behavioral Models}
\author{
    Greg d'Eon\textsuperscript{\rm 1},
    Sophie Greenwood\textsuperscript{\rm 1,\rm 2}\thanks{Work done while at the University of British Columbia.},
    Kevin Leyton-Brown\textsuperscript{\rm 1},
    James R. Wright\textsuperscript{\rm 3}
}
\affiliations{
    \textsuperscript{\rm 1}University of British Columbia\\
    \textsuperscript{\rm 2}Cornell University\\
    \textsuperscript{\rm 3}University of Alberta\\
    gregdeon@cs.ubc.ca, sjgreenwood@cs.cornell.edu, kevinlb@cs.ubc.ca, jrwright@ualberta.ca 
}

\begin{document}

\maketitle

\begin{abstract}
Researchers building behavioral models, such as behavioral game theorists, use experimental data to evaluate predictive models of human behavior. 
However, there is little agreement about which loss function should be used in evaluations, with error rate, negative log-likelihood, cross-entropy, Brier score, and squared L2 error all being common choices. 
We attempt to offer a principled answer to the question of which loss functions should be used for this task, formalizing axioms that we argue loss functions should satisfy. 
We construct a family of loss functions, which we dub ``diagonal bounded Bregman divergences'', that satisfy all of these axioms. 
These rule out many loss functions used in practice, but notably include squared L2 error; we thus recommend its use for evaluating behavioral models.
\end{abstract}

\section{Introduction}
Theoretical models of decision-making are often poor descriptions of behavior in practice.
As a prime example, classic economic models such as Nash equilibrium fail to describe salient aspects of human behavior:
people often choose dominated actions~\citep{Goeree2001}
and fail to account for others' strategic decision making~\citep{Kneeland2015}.
In response to such failures, fields such as behavioral game theory aim to develop interpretable models that can predict human responses to strategic situations.
Such models are helpful to cognitive scientists, for learning how humans think when confronted with economic or strategic choices; to designers of economic systems, for tuning these systems to perform better in practice; and to designers of cooperative AI agents, for enabling these agents to effectively coordinate their behavior with humans~\cite{Hu2020, Carroll2019}.

% need losses
However, evaluating the quality of such a model on a dataset requires a loss function.
Researchers working in behavioral game theory have made a wide variety of different choices about precisely which loss function to use for such evaluations, with error rate, negative log-likelihood, cross-entropy, and (at least two notions of) mean-squared error all being common choices.
Clearly, the choice is a substantive one, as different losses will disagree about the quality of a prediction.
Which loss function should they use? 

In this paper, we attempt to answer this question with a first-principles argument.
Though we are motivated by behavioral game theory---and so it is the basis of our examples---
our argument depends only on four key characteristics of this field.
First, there is some mapping of interest from settings to \emph{distributions} over \emph{finite sets of discrete outcomes}
(e.g., the distribution of human decisions in strategic situations).
Second, it is possible to collect \emph{multiple samples} from this mapping for any given setting 
(e.g., by running an experiment with multiple participants).
Third, a researcher seeks a \emph{predictive model} of this mapping, which can predict the distribution of unseen data.
Fourth, this model must also be \emph{interpretable}, having few parameters whose values can be inspected and understood, and so it cannot generally represent the true mapping perfectly.
Our arguments can therefore be extended to other domains that share these characteristics;
we give several examples at the end of this paper.

% axioms
From these characteristics, we argue that loss functions should satisfy five key axioms.
The first two, which we call \textit{alignment} axioms, ensure that the loss function induces a correct preference ordering over predictions.
These axioms, \textit{sample Pareto-alignment} and \textit{distributional Pareto-alignment}, ensure that the loss function penalizes predictions that are clearly worse (on a given dataset or in expectation over realizations of this data, respectively).
The other three, \textit{interpretability} axioms, relate the numerical value of the loss to a prediction's quality.
\textit{Empirical distribution sufficiency} requires that the loss be invariant to the number or order of the observations; 
\textit{counterfactual Pareto-regularity} ensures that the loss appropriately respects changes in the data; 
and \textit{zero minimum} gives the loss an interpretable optimum.

% result
We show that it is possible to satisfy all of these axioms: we identify an entire family of loss functions that do so, which we dub ``diagonal bounded Bregman divergences''.
Exactly one widely used loss function, the squared L2 error between the predicted and empirical distributions, belongs to this set; we show how each of the other common loss functions violates at least one axiom.
In particular, the entire class of scoring rules,\footnote{
    The term ``scoring rule'' has multiple definitions in the literature. 
    We use a standard definition~\cite[e.g.,][]{Savage1971, Gneiting2007} that a scoring rule computes a loss separately for each observation, then takes the mean of these losses (\Cref{def:scoring-rule}).
    Other authors~\cite[e.g.,][]{Abernethy2012} use the term to refer to any arbitrary loss function. %, including all of the loss functions we discuss.
    Of course, our results on scoring rules only apply to the former, more restrictive definition.
}
a class of loss functions with celebrated alignment properties, all fail our interpretability axioms, making them suitable for training models but not evaluating them.

\longver{
    \subsection{Related Work}
    Before we begin, we review related work on loss functions for evaluating probabilistic predictions through the lenses of statistics, forecasting, and economics.
} 

\emph{The statistician's view: the likelihood principle.}
It might seem that the problem of choosing a loss function is a straightforward application of statistical inference: 
given a dataset and a model class that induces a set of probability distributions, we seek to understand how well each distribution describes the data.
Then, the standard statistics textbook argument is that we should use the likelihood of the data to evaluate each of these predicted distributions.
This argument is known as the ``likelihood principle'' \citep[e.g.,][]{Berger1988}: if the data was generated by one of the predicted distributions, then likelihood is a sufficient statistic for this distribution.
The catch is that this argument relies on the assumption that the model class is ``well-specified'', containing a model that outputs the true generating distribution.
This is not usually the case when evaluating interpretable models, which typically approximate behavior rather than to predict it perfectly.
We elaborate further on the problem of evaluating misspecified models when presenting our alignment axioms.

\emph{The forecaster's view: scoring rules.}
Another closely related problem is that of evaluating probabilistic forecasts of future events.
Work in this field generally uses \emph{scoring rules} \citep[e.g.,][]{Gneiting2007}, a class of loss functions that evaluate predictions independently on each observation.
Axiomatic characterizations from this literature agree that losses should be \emph{proper}---
the expected loss should be minimized by the true distribution, an axiom that we refer to in our analysis as ``distributionally proper''
---but diverge beyond this point:
negative log-likelihood is the only proper scoring rule that satisfies a locality axiom~\citep{McCarthy1956}, 
and two different neutrality axioms characterize Brier score~\citep{Selten1998} and the spherical score~\citep{Jose2009}.
Our work differs in that we propose axioms that address critical problems that arise when evaluating behavioral models, without being concerned that we are left with an entire class of loss functions.

Some authors have proposed stronger alternatives to propriety. 
Instead of simply requiring that the correct prediction minimize the expected loss, others have considered 
lower-bounding the loss of incorrect predictions~\cite{Friedman1983, Nau1985, Haghtalab2019},
maximizing the loss of a naive prediction~\cite{Li2022},
or ensuring that it also receives a lower loss in finite samples with high probability~\cite{Haghtalab2019}.
These axioms focus on identifying correct predictions, 
while we focus on comparing and evaluating incorrect predictions.

The field of property elicitation extends the definition of propriety in a different way, 
aiming to construct loss functions whose expectations are minimized at other summary statistics of a distribution;
propriety is the special case of eliciting the mean.
Of particular interest here is work on eliciting multiple properties~\cite{Lambert2008, Fissler2019},
as their ``accuracy rewarding'' and ``order sensitivity'' axioms are similar to our alignment axioms.
We discuss this relationship further in Section~\ref{sec:axioms}.

\textit{Evaluating model classes.}
Our axioms are concerned with evaluating individual predictions. 
\citet{Fudenberg2021} tackle the related problem of evaluating a {\it model class}, considering the cross-validation performance of a training algorithm that selects a model from this class.
They formalize a \emph{completeness} metric, which transforms an existing loss, giving a score of 100\% to an algorithm with the best possible cross-validation performance and 0\% to a baseline algorithm.
Their work complements ours: 
their completeness measure can be applied to any loss function, 
but they do not claim how this loss should behave on individual datasets. 
We thus recommend that researchers evaluating a model class should apply completeness to a loss that satisfies our alignment axioms.

\section{Setup and Existing Losses}
\label{sec:existing-losses}

We now give a formal description of the problem.
We start by making a simplification.
While researchers generally collect data and evaluate models on many different settings (e.g., games) at once, reporting a model's aggregate performance across these settings, we focus on evaluation in a \textit{single} setting.
However, this simplified analysis is useful: 
any loss that behaves appropriately on an arbitrary number of settings must behave appropriately in the special case of a single setting, so all of the loss functions we disqualify are also unsuitable for multiple settings.
We discuss the multiple-setting case in detail in Appendix~\ref{appendix:multiplegames}, where we provide straightforward extensions of our axioms and results.

We model a single scenario as follows.
Let $A = \{1, \ldots, d\}$ be a fixed set of choices available to the decision maker being modelled (e.g., actions available to experiment participants), and let $\Delta(A)$ be the set of distributions over these choices, i.e., the $(d-1)$-dimensional simplex.
We assume that there exists a fixed but unknown true distribution $p \in \Delta(A)$ of behavior, where the randomness in $p$ captures both differences between individuals and randomness in their behavior.
An analyst can collect a dataset consisting of $n$ independent, identical draws from $p$, which we denote $y \sim p^n$, representing actions taken by distinct actors (for example, different participants in a psychology experiment).
We denote the set of all such datasets by $\Datasets(A) = \bigcup_{n=1}^\infty A^n$.

The analyst is equipped with a model class, which induces a set of predicted distributions $\mathcal{F} \subseteq \Delta(A)$.
As this model class is interpretable (e.g., a parametric model with few parameters), this inequality can generally be strict, and $\mathcal{F}$ does not generally include the true distribution $p$.
Their goal is then to choose a model from this class that is good at predicting the distribution of behavior on unseen data.\footnote{We use the terms ``model'' and ``prediction'' interchangeably, as the model is only used to predict behavior in a single scenario.}
To make their choice, the analyst relies on a loss function $\Obj: \Delta(A) \times \Datasets(A) \to \R$ representing preferences over these predictions: that is, $\Obj(f, y) < \Obj(g, y)$ if and only if $f$ is a better description of the data than $g$.
Note that our analysis can easily be modified to handle objective functions that are expressed in a ``positive'' sense: for example, it is equivalent to maximize accuracy or minimize error rate.

We pause to define some additional notation.
For any dataset $y \in \Datasets(A)$, let $n(y)$ denote the number of observations in $y$ (or simply $n$, when $y$ is clear from context), and let $\bar p(y) \in \Delta(A)$ be its empirical distribution: that is, for all $a \in A$, \smash{$\bar p(y)_a = \nicefrac{\sum_{i=1}^{n(y)} \mathds{1}_{\{y_i = a\}}}{n(y)}$}.
Lastly, for any action $a \in A$, let $e_a \in \Delta(A)$ denote a point mass distribution on $a$. 

\longver{
    \subsection{Common Loss Functions}
}

While behavioral game theorists broadly take this approach of evaluating their models with \emph{some} loss function, they largely disagree about precisely \emph{which} loss function to use;
in fact, it is not uncommon for a single paper to use multiple different losses while analyzing different experiments.
To illustrate this disagreement, we give seven examples of losses that are common in the literature.

% family: error rate
First, one common choice is the \textbf{error rate}~\citep{Fudenberg2019, GarciaPola2020}.
It is especially common when $\mathcal{F}$ consists only of deterministic predictions, which assign probability to one action.
\begin{align*}
    \textstyle \ErrorRate(f, y) &= \textstyle\sum_{a=1}^d \bar p(y)_a (1 - f_a),
\end{align*}
It is similar to \textbf{mean absolute error} (\MAE)~\citep{Camerer2004, Levin2019}.
\begin{align*}
    % = \frac{1}{n} \sum_{i=1}^n (1 - f_{y_i}) 
    \textstyle \MAEloss(f, y) &= \|f - \bar p(y)\|_1 = \textstyle \sum_{a=1}^d |f_a - \bar p(y)_a|.
\end{align*}
These two losses are attractive because of their clearly defined scale, with a loss of 0 being achieved by a prediction that never makes mistakes (error rate) or matches the data perfectly (\MAE), and a maximum loss of 1 or 2, respectively, by a prediction that is never correct.

% Family: likelihood
Next, several common losses are based on the likelihood of the data, given the prediction.
Perhaps the most common choice of loss in all of behavioral game theory is \textbf{negative log-likelihood} (NLL)~\citep{McKelvey1992,Stahl1995,Wright2017}.
\begin{align*}
    \textstyle \NLL(f, y) 
    % = -\sum_{i=1}^n \log(f_{y_i}) 
    = -n \sum_{a=1}^d \bar p(y)_a \log(f_a).
\end{align*}
\textbf{Cross-entropy}~\citep{Kolumbus2019} differs from NLL by a factor of $n$, and \textbf{KL divergence}
further subtracts the entropy of the dataset.
\begin{align*} 
  \textstyle \CrossEntropy(f, y) &= \textstyle \frac{1}{n} \NLL(f, a), \\
  \textstyle \KL(f, y) &= \textstyle -\sum_{a=1}^d \bar p(y)_a \log(\frac{f_a}{\bar p(y)_a}).
\end{align*}
All three of these options are rooted in statistics: they make up the core of many statistical hypothesis tests, and all three of them agree with the likelihood principle.

% Family: squared error
Two more losses originate from regression problems and forecasting.
One is the \textbf{Brier score}, frequently referred to as mean-squared error or mean-squared deviation~\citep{Camerer2004,Golman2019}.
\begin{align*}
    \textstyle \Brier(f, y)
    = \frac{1}{n} \sum_{i=1}^n \|f - e_{y_i}\|_2^2.
    % = \sum_{a=1}^d \bar p(y)_a ( (1-f_a)^2 + \sum_{a' \neq a} f_{a'}^2 ). 
\end{align*}
A small modification is the \textbf{squared L2 error}, which is often also called MSE or MSD~\citep{Camerer2003, Selten2008}.
\begin{equation*} \label{eqn:loss-l2}
    \textstyle \SquaredLTwo(f, y)
    = \|f - \bar p(y)\|_2^2
    = \sum_{a=1}^d (f_a - \bar p(y)_a)^2.
\end{equation*}
Both are natural options for researchers familiar with regression problems, where it is typical to optimize a least-squares objective.
They also have roots in forecasting, as the Brier score was originally introduced for evaluating weather forecasts~\citep{Brier1950}.
We avoid the common but ambiguous term ``mean-squared error'' to avoid confusion.

%\0 Scoring rules 
Finally, a unifying definition that ties together many losses is the concept of a scoring rule.
\begin{definition} \cite[][page 2.]{Gneiting2007}
    \label{def:scoring-rule}
    A \emph{scoring rule} is a function $S: \Delta(A) \times A \to \R$ that maps a prediction $f \in \Delta(A)$ and a single outcome $a \in A$ to a score $S(f, a)$.
    By averaging these scores over the dataset, every scoring rule $S$ induces a loss function
    $\Obj_S(f, y) = \frac{1}{n} \sum_{i=1}^n S(f, y_i) = \sum_{a \in A} \bar p(y)_a S(f, a)$.
\end{definition}
Scoring rules are popular due to their simple functional form, which evaluates the prediction independently on each observation.
Their alignment properties are also the subject of several celebrated results~\citep{Savage1971,Gneiting2007}, which we describe in detail in \Cref{sec:revisitexisting}.
Error rate, negative log-likelihood, cross-entropy, and Brier score are scoring rules; \MAE, KL, and squared L2 are not.

\section{Formalizing an Ideal Loss Function}
\label{sec:axioms}
Each loss function from the previous section captures the quality of a prediction on a dataset with a single number, inducing preferences over these predictions. 
Of course, these loss functions will not always agree with each other about how to order different predictions.
Is each loss an equally acceptable choice?
To answer this question, we turn to an axiomatic analysis, formalizing axioms that a loss function in a behavioral setting ought to obey. 
We aim to identify axioms that are as weak as possible, only disqualifying loss functions that exhibit clearly objectionable behavior.

Our axioms can be grouped according to two distinct roles that a loss function serves in describing the quality of a prediction.
First, loss functions are used to compare models within a fixed experimental setting.
This occurs both during training, when a modeller aims to minimize expected loss on future data; and when evaluating models on a given dataset, comparing losses to see which model achieves the best performance. 
Our \emph{alignment} axioms address this case, requiring that the loss correctly orders predictions in cases where quality disparities are unambiguous; both are extensions of already standard \emph{propriety} axioms.
Second, loss functions are used to understand model performance more broadly; studies report losses and these values are interpreted as conveying information about how well a given model captured human behavior.
Our \emph{interpretability} axioms ensure that the loss can indeed be understood in this way, having a well-defined reference point and changing coherently as the data varies.

\longver{\subsection{Alignment Axioms}}
\shortver{\textbf{\textit{Alignment axioms.}}}
Our first alignment axiom pertains to the training process.
While training a predictive model, a modeller's goal is to select a prediction that has low expected test loss over new, unseen data.
Thus, if one model better fits the data than another,
it should receive a lower expected loss.

What do we mean by ``better''?
Reasonable people disagree about many comparisons between models, but some are unarguable.
For instance, a perfect prediction---one that exactly matches the data generating process---is better than an imperfect one.
A standard axiom known as Propriety captures this intuition, requiring that a perfect prediction minimizes the expected loss. 
To distinguish it from a Sample Propriety axiom that will follow, we refer to it as \emph{Distributional Propriety}.
\begin{axiom}[Distributional Propriety (DP)] For all predictions $f \in \Delta(A)$ and all $n \ge 1$, $p \in \Delta(A)$, \begin{kequation*}
    f \neq p \implies \E_{y \sim p^n}\Obj(p, y) < \E_{y \sim p^n}\Obj(f, y).
\end{kequation*}
\end{axiom}

Unfortunately, Distributional Propriety is insufficient for interpretable models: there is often no model in a given class that is able to output an arbitrary distribution.
We thus impose a stronger requirement that implies Distributional Propriety: that we should prefer one (potentially imperfect) prediction to another whenever the first is an unambiguously better fit. 
We formalize this idea with the notion of a Pareto improvement, which we will use extensively in what follows.
\begin{definition}[Pareto improvement]
    Let $p, q, r \in \Delta(A)$ be three distributions. 
    We say that $q$ is a \emph{Pareto improvement} over $p$ with respect to $r$, denoted by $q \succ_r p$, if for all $a \in A$, either $p_a \leq q_a \leq r_a$ or $p_a \geq q_a \geq r_a$, and furthermore this inequality between $p_a$ and $q_a$ is strict for at least one $a$.
\end{definition}
In other words, $q$ is a Pareto improvement over $p$ if $q$ is at least as close to $r$ as $p$ in every dimension, and strictly closer to $r$ in some dimension. 
Then, if one prediction is a Pareto improvement over another with respect to the true distribution---i.e., its predicted probabilities are uniformly closer to the truth---it should receive a lower expected loss.
\begin{axiom}[Distributional Pareto-Alignment (DPA)]
    For all predictions $f, g \in \Delta(A)$, $n \ge 1$, and $p \in \Delta(A)$,
    \begin{kequation*}
        f \succ_p g \implies \E_{y \sim p^n}\Obj(f, y) < \E_{y \sim p^n}\Obj(g, y).
    \end{kequation*}
\end{axiom}

A similar axiom was proposed by \citet{Lambert2008} under the name ``accuracy-rewarding'', and by \citet{Fissler2019} under the name ``order sensitive''.
There is only one difference: in their settings, a prediction is a vector in $\R^d$, containing independent predictions for $d$ different summary statistics of the dataset.
Because our predictions lie on the simplex, they are not independent in this way: e.g., predicting that one action has a probability of 1 constrains the predictions for all other actions to be 0.

Next, we consider the situation where two models' predictions are compared to each other on a fixed dataset.
This, too, is a fundamental step in behavioral modelling: to evaluate a proposed model, one must compare its predictions to other existing models on \emph{some} dataset to understand whether their proposal better captures human behavior.
Here, if one model fits the data better than another, it should receive a lower loss.

As with DP, it is standard to insist that the loss must be minimized when the empirical distribution is reported.
\begin{axiom}[Sample Propriety (SP)]
    For all predictions $f \in \Delta(A)$ and sampled datasets $y \in \Datasets(A)$, 
    \begin{kequation*}
        f \neq \bar p(y) \implies \Obj(\bar p(y), y) < \Obj(f, y).
    \end{kequation*}
\end{axiom}

As above, though, Sample Propriety is insufficient for interpretable models.
In this case, it is necessary to prefer predictions that are clearly closer to the empirical distribution, accurately reflecting improvements even away from the optimum.
We capture this intuition with a second alignment axiom, which we refer to as Sample Pareto-Alignment.
\begin{axiom}[Sample Pareto-Alignment (SPA)]
    For all predictions $f, g \in \Delta(A)$ and sampled datasets $y \in \Datasets(A)$, 
    \begin{kequation*}
        f \succ_{\bar p(y)} g \implies \Obj(f, y) < \Obj(g,y).
    \end{kequation*}
\end{axiom}
In the same way as DPA implies DP, SPA implies SP.

\longver{\subsection{Interpretability Axioms}}
\shortver{\textbf{\textit{Interpretability axioms.}}}
Our alignment axioms constrain how the loss may vary as the prediction varies.
Our next axioms constrain how the loss may vary as the data varies. 
Such constraints are important for ensuring that loss represents an understandable measurement of a prediction's quality.

Because it is possible to evaluate a model on multiple observations, one simple way that the data could be changed is simply by observing the same empirical distribution with a different set of observations.
This could happen if an experimenter made the same observations in a different order, or collected twice as many observations.
Since each observation is independent (e.g., representing an independent trial with a distinct participant), we argue that the loss should be unaffected by such changes to the data.
\begin{axiom}[Empirical Distribution Sufficiency (EDS)]
    For all datasets $y, y' \in \Datasets(A)$ and predictions $f \in \Delta(A)$, 
    $\bar p(y) = \bar p(y') \implies \Obj(f, y) = \Obj(f, y')$.
\end{axiom}
This implies a weaker axiom of \emph{exchangeability}---that permuting the observations does not affect the loss---which is a standard assumption in statistics~\citep[e.g.,][]{Easton1989}.

% CPR
What if the dataset varies in a more substantial way?
For example, one might replicate an experiment with another group of participants, producing a new set of observations for the same setting, or run slight variations to an experiment to assess their impact on the quality of a model \citep{Goeree2001}.
In both cases, it would be undesirable if the change in the data could cause the prediction to clearly decrease in quality, but be awarded a better loss. 

As with varying predictions, there are many ways in which datasets could vary for which reasonable people could disagree about whether the same prediction ought to receive a higher or lower loss.
However, we can again leverage the insight that Pareto improvements are unambiguously better: 
holding a prediction fixed, if the empirical probabilities of the data are brought closer to the predictions for at least some actions and further in none, it is clear that this dataset is better described by the prediction. 
In such cases, we require that the loss must also improve.
\begin{axiom}[Counterfactual Pareto-Regularity (CPR)]
    Let $f \in \Delta(A)$ be a fixed prediction.
    Suppose that $y, y' \in \Datasets(A)$ are two datasets of equal size, where $n(y) = n(y')$. Then \begin{kequation*}
        \bar p(y) \succ_{f} \bar p(y') \implies \Obj(f, y) < \Obj(f, y').
    \end{kequation*}
\end{axiom}
Note that this axiom leverages the discrete outcome space, 
as it does not obviously generalize to arbitrary distributions.

% Normalization
Up to this point, all of the axioms have only described equalities or inequalities between certain pairs of losses.
None have constrained the precise numerical values of the losses: indeed, if $\Obj$ satisfies all of these axioms, then any positive affine transformation $a\Obj + b$ (with $a > 0$) does too.
This leaves users with a choice of how to set these two degrees of freedom.
We propose to use this freedom to constrain the minimum loss, requiring that a perfect prediction achieves a loss of zero (which must be the loss function's minimum, by SP).
This makes the loss easier to interpret: when the analyst has multiple observations in the same setting, it removes the possibility for irreducible error, where even a perfect prediction could get a positive loss.

\begin{axiom}[Zero-Minimum (ZM)]
    For all $y \in \Datasets(A)$, $\Obj(\bar p(y), y) = 0$.
\end{axiom}

ZM is admittedly the most subjective of our axioms: 
for example, on some problems, it might be reasonable to anchor the loss to a different baseline, such as a uniform random prediction.
However, its addition is inconsequential when analyzing existing loss functions: in \Cref{sec:revisitexisting}, we show that each commonly used loss that violates ZM also violates CPR.

\section{Diagonal Bounded Bregman Divergences}\label{sec:dbbds}
% outline
With these desiderata in mind, the obvious question is: are there loss functions that satisfy all of our axioms? 
In this section, we provide a positive answer.
We first appeal to existing results to show that even asking for a subset of the axioms gives these loss functions considerable structure: Bregman divergences are essentially the only losses that satisfy SP, DP, EDS, and ZM.
Narrowing down this class further, we identify a family of losses, which we coin \emph{diagonal bounded Bregman divergences}, that each satisfy our whole set of axioms (SPA, DPA, CPR, EDS, and ZM).

% Bregman divergences
Let us now make these claims more precise. We first define a Bregman divergence. Let $\overline{\R}$ denote the extended real numbers $\R \cup\{\pm \infty\}$, and adopt the convention that $0 \cdot \infty = 0$. 
% subgradient
\begin{definition}\label{def:subgradient} Let $B: C \to \R$ be a closed and proper strictly convex function on a convex set $C \subseteq \R^k$. 
    Then a \emph{subgradient} of $B$ is a function \smash{$dB: C \to \overline{\R}^k$} such that 
    \[B(x) - B(x_0) \geq dB(x_0)^T (x - x_0)\]
    for all $x_0, x \in C$.
    If $B$ is also differentiable, it has a unique subgradient $\nabla B$ on the interior of $C$.
\end{definition} 

% Bregman divergence
\begin{restatable}{definition}{defbregmandivergence}    
    Given a closed and proper strictly convex function $B: C \to \R$ and subgradient $dB$ of $B$, the \emph{Bregman divergence} \smash{$\nabla_{(B, dB)}: C\times C \to \overline{\R}_{\geq 0}$} of $B$ and $dB$ is 
    \[\nabla_{(B, dB)} (p, q) = B(p) - B(q) - dB(q)^T (p-q).\]
\end{restatable}
We now leverage existing work from the field of property elicitation.
\citet{Abernethy2012} show that essentially all loss functions satisfying DP are equivalent to Bregman divergences between a summary statistic of the dataset and the prediction, up to a translation by a function of the data.
This immediately yields the following result.

\begin{theorem}[Corollary of Theorem 11 of \citet{Abernethy2012}, informal]\label{thm:af12}
    For any $n$, 
    under mild technical conditions, a loss function $\Obj$ that satisfies DP must be of the form \begin{kequation*}
    \Obj(f,y) = \nabla_{(B, dB)}(\rho(y), f) + c(y)
    \end{kequation*}
    for some closed and proper strictly convex function $B$,
    subgradient $dB$ of $B$, 
    translation $c: A^n \to \R$,
    and summary statistic $\rho: A^n \to \Delta(A)$, 
    where \smash{$\E_{y\sim p^n} \rho(y) = p$} for all $p$.
\end{theorem}

We extend this result, showing that the SP and ZM axioms additionally determine $c$ and $\rho$, and that the EDS axiom removes the dependence on $n$.
In other words, essentially every loss function satisfying DP, SP, ZM, and EDS is a Bregman divergence between the empirical distribution and the prediction.
\begin{theorem}[Informal]\label{thm:bregman_characterization_informal}
Under mild technical conditions, 
a loss function $\Obj$ satisfies SP and DP if and only if \begin{kequation*}\label{eqn:spdpbd}
    \Obj(f,y) = \nabla_{(\mathcal{B}_n, d\mathcal{B}_n)}(\bar p(y), f) + c(y)
\end{kequation*}
for some family of closed and proper strictly convex functions $\mathcal{B}$ 
with subgradients $d\mathcal{B}$
and some translation $c$.
Additionally, $\Obj$ satisfies ZM if and only if $c(y) = 0$ for all y,
and $\Obj$ further satisfies EDS if and only if there is some convex function $B$ and subgradient $dB$ such that $\nabla_{(\mathcal{B}_n, d\mathcal{B}_n)} = \nabla_{(B, dB)}$ for all $n$.
\end{theorem}

We defer a formal statement and proof of \Cref{thm:bregman_characterization_informal} to \Cref{appendix:bregman_characterization}, as describing the technical conditions on $\Obj$ takes care.
The proof obtains $\Obj$ satisfying DP from Theorem 11 of \citet{Abernethy2012}, then applies standard facts about Bregman divergences to show that the additional axioms constrain $\rho$, $c$, and $\mathcal{B}$ as described. 
The reverse direction follows from standard observations from convex analysis. 

However, not all Bregman divergences satisfy our remaining axioms SPA, DPA, and CPR. 
For example, taking \smash{$B(f) = \sum_{a=1}^d f_a \log f_a$} recovers the KL divergence; we will show in Section \ref{sec:revisitexisting} that this does not satisfy SPA. 
Our main result is that all of our axioms are satisfied by the restricted set of \emph{diagonal bounded Bregman divergences}.
\begin{restatable}[Diagonal bounded Bregman divergence (DBBD)]{definition}{dbbddefn}
    Let $b: [0,1] \to \R$ be a continuously differentiable convex function where $b'$ is bounded on $[0, 1]$. Let $B_b(x) = \sum_i b(x_i)$ for $x \in [0,1]^d$. 
    Then, a \emph{diagonal bounded Bregman divergence} is a loss function $\Obj: \Delta(A) \times \Datasets(A) \to \R$, where 
    $\Obj(f,y) = \nabla_{(B_b, \nabla B_b)}(\bar p(y), f).$
\end{restatable}

\begin{restatable}{theorem}{dbbd}\label{thm:dbbd}
    If $\Obj$ is a DBBD, then $\Obj$ satisfies SPA, DPA, EDS, CPR, and ZM.
\end{restatable}

We again defer the proof to \Cref{appendix:dbbd_proof}.
Briefly, EDS is trivial; ZM follows from Theorem \ref{thm:bregman_characterization_informal}; SPA, DPA, and CPR leverage the diagonal structure and convexity of $B_b$.

\section{Evaluating Existing Loss Functions}\label{sec:revisitexisting}
We now revisit the loss functions introduced in Section \ref{sec:existing-losses}. 
It is straightforward to see that squared L2 error is a DBBD (with $b(x) = x^2$) and so it satisfies all of the axioms. 
Each other loss function violates at least one axiom (Table~\ref*{tab:existing-losses-axioms}).
We give an example for each loss below, showing that each axiom violation leads to undesirable results under reasonable conditions.
We also demonstrate many of these axiom violations on real behavioral data in Appendix~\ref{appendix:real_data}.

\begin{table*}[t]
    \centering
    \begin{tabular}{rcccccccc}
        Axiom & 
        \rot{Error rate} &
        \rot{\MAE} &
        \rot{NLL} & 
        \rot{Cross-entropy} & 
        \rot{KL divergence} & 
        \rot{Brier score} & 
        \rot{Squared L2 error} \\
        \midrule
        %\multicolumn{5}{c}{Common loss functions (Section~\ref{sec:existing-losses})} \\
        Sample Pareto-Alignment (SPA)  & 
         \no & \yes & \no & \no & \no & \yes & \yes \\
        Sample Propriety (SP)  & 
        \no & \yes & \yes & \yes & \yes & \yes & \yes \\
        Distributional Pareto-Alignment (DPA) & 
        \no & \no & \no & \no & \no & \yes & \yes \\
        Distributional Propriety (DP) & 
        \no & \no & \yes & \yes & \yes & \yes & \yes \\
        Empirical Distribution Sufficiency (EDS) & 
        \yes & \yes & \no & \yes & \yes & \yes & \yes \\
        Counterfactual Pareto-Regularity (CPR) &
        \no & \yes & \no & \no & \no & \no & \yes \\
        Zero Minimum (ZM) &
        \no & \yes & \no & \no & \yes & \no & \yes \\
        \bottomrule
    \end{tabular}
    \caption{Existing losses and their status under the axioms.}
    \label{tab:existing-losses-axioms}
\end{table*}

\longver{\subsection{Error rate}}
\shortver{\paragraph{Error rate.}} 
Error rate violates every axiom except EDS. 
We show that error rate violates both SP and ZM with the following example.
\begin{example}\label{ex:err} 
    Consider a game in which a player can choose between two actions, ``defect'' and ``cooperate''. 
    Suppose that in the true distribution of human play, two-thirds of players defect: $p = (2/3, 1/3)$.
    In an experiment with 10 distinct participants, an analyst finds that 6 chose to defect, while the remaining 4 chose to cooperate, yielding an empirical distribution of $\bar p(y) = (0.6, 0.4)$. Letting $(f, 1-f)$ be a prediction in this setting, the error rate on this dataset is 
    \[
        \textstyle
        \ErrorRate(f,y) = 1 - 0.6f - 0.4(1-f) = 0.6 - 0.2f.
    \]
    This expression is minimized by the prediction $f = 1$, which has an error rate of $0.4$.
    In particular, this prediction achieves a lower error rate than reporting the empirical distribution, which has an error rate of $\ErrorRate(\bar p(y), y) = 0.48$. 
    %In other words, a model that does a much poorer job of predicting the data gets a lower loss.
\end{example}
This example illustrates a general problem: for any dataset, the error rate is minimized by predicting the mode, giving more credit to predictions that overestimate the probability of the most likely action.
    
\longver{\subsection{Mean absolute error}}
\shortver{\paragraph{Mean absolute error.}}
\MAE\ satisfies both SPA and ZM, but does not satisfy DPA or DP.
In some cases, a model that predicts the true population distribution gets worse expected \MAE\ on unseen data than an incorrect prediction.
\begin{example} 
    Suppose, as in Example \ref{ex:err}, that the true distribution is $p = (2/3, 1/3)$.
    However, now suppose that the dataset is not yet available; all that is known is that it consists of 10 independent observations sampled from $p$. 
    Then, the expected loss of predicting $(f, 1-f)$ is 
    $2\E_{y \sim p^{10}}|f-\bar p(y)_D|,$
    where $10\bar p(y)_D$, the number of participants that defect, is a Binomial random variable with parameters $n = 10, p = 2/3$. 
    This expected loss is minimized by predicting the median of $\bar p(y)_D$, which is $0.7$. 
    In particular, this prediction receives an expected loss of $0.235$, which is lower than the expected loss of $0.243$ achieved by predicting the true distribution.
\end{example}
This example, too, generalizes: 
in any setting with two actions, the expected loss is minimized by reporting the median of the empirical probability distribution, 
which is generally not equal to $p$.
In other words, if a model is designed to minimize expected loss, \MAE\ fails to elicit the true distribution.

\longver{\subsection{Negative log-likelihood}}
\shortver{\paragraph{Negative log-likelihood.}}
NLL is the only loss that violates EDS, which we show in the following example.
\begin{example} \label{ex:nll}
    A second experimenter attempts to reproduce the results from Example~\ref{ex:err}.
    They first fit a model to the existing dataset $y$, which has an empirical distribution of $\bar p(y) = (0.6, 0.4)$.
    Their model fits perfectly, returning the exact empirical distribution and getting a negative log-likelihood of $\NLL(\bar p(y), y) = 2.9$.
    They then collect their own dataset $y'$, re-running the experiment with a different set of 20 participants; they find that 12 defect and 8 cooperate, resulting in the same empirical distribution.
    Although their model still fits the data perfectly, they are surprised to see that it now receives a higher loss of $\NLL(\bar p(y'), y') = 5.8$.
\end{example}
In general, negative log-likelihood scales linearly with the number of observations in the dataset, as it takes a sum over the observations rather than an average.

\longver{\subsection{Cross-entropy and Brier score}}
\shortver{\subsection{Cross-entropy and Brier score.}}
We group the next two losses together as they suffer from the same key issue: they violate both CPR and ZM.
\begin{example}
    Undeterred, our experimenter from Example~\ref{ex:nll} considers different loss functions.
    Using Brier score and cross-entropy to evaluate their perfect model on the original dataset, they obtain losses of
    \[
        \Brier(\bar p(y), y) = 0.48; \quad
        \CrossEntropy(\bar p(y), y) = 0.29.
    \]
    They collect a third dataset $y''$; these 10 participants are quite different, with 9 defecting and only one cooperating.
    They are surprised to find that, despite failing to predict this new dataset perfectly, their model receives lower losses of
    \[
        \Brier(\bar p(y), y'') = 0.36; \quad
        \CrossEntropy(\bar p(y), y'') = 0.24.
    \]
\end{example}
The first dataset in this example demonstrates violations of ZM: there is no indication that the model has made a perfect prediction, leaving it unclear to the experimenter whether there is room for improvement. 
In general, the both losses have a non-zero minimum as long as the dataset has two distinct observations.
The second dataset shows violations of CPR: it intuitively appears that the model is now better, even though it no longer outputs the correct distribution.

\longver{\subsection{KL divergence}}
\shortver{\paragraph{KL divergence.}}
The KL divergence is a translated version of cross-entropy that satisfies ZM, but not SPA, DPA, or CPR.
The key issue is that KL divergence gives infinite losses at the boundary.
That is, when a model predicts that an action has zero probability of being selected, but the action is observed in the data, that model will have an infinite KL divergence.
This leads to situations such as the following.
\begin{example}
    Now, suppose that there are three actions, with a true distribution of $p = (0.001, 0.199, 0.8)$, and that among $100$ participants we observe $y = (1, 19, 80)$, yielding an empirical distribution of $\bar p(y) = (0.01, 0.19, 0.80)$.
    Consider comparing two predictions on this dataset: the very coarse prediction of $f = (0, 1, 0)$ and the far more precise $f' = (0, 0.2, 0.8)$. 
    Although $f'$ is a better prediction, as it is closer to $\bar p(y)$ than $f$ on both the second and third actions, both receive equal losses of $\KL(f, y) = \KL(f', y) = \infty$.
\end{example}
In general, when every action appears at least once in the dataset, KL divergence assesses every prediction that places 0 probability on any action as equally bad, and considers all of these predictions to be worse than any prediction having full support.
This is a serious problem, as it is common for every action to be played at least once in sufficiently large behavioral datasets.
This makes it difficult to evaluate classical economic predictions, such as Nash equilibrium, which assign 0 probability to many actions. To avoid this issue, some researchers \cite[e.g.,][]{Stahl1994} perturb the predictions of such models to yield finite losses, but in doing so introduce an important new parameter and sacrifice the ability to evaluate the original models.

\longver{\subsection{Scoring rules}}
\shortver{\paragraph{Scoring rules.}}
Recall that error rate, cross-entropy, negative log-likelihood, and Brier score
each violated the ZM and CPR axioms. 
It turns out that these failures are common to all scoring rules, implying that scoring rules should not be used to report model performance.

\begin{restatable}{proposition}{srnegative}\label{thm:scoring-rules-negative-result}
    Every scoring rule that satisfies SPA violates ZM. Moreover, no scoring rule satisfies CPR.
\end{restatable}

We defer the proof to the appendix. 
Intuitively, since scoring rules must consider each sample independently, they must treat every sample as if it were the entire dataset.
Then, in order to satisfy SPA, scoring rules must give positive losses to every nondeterministic prediction, causing them to violate the ZM axiom.
Moreover, scoring rules are linear in the empirical probabilities $\bar p(y)$ (\Cref{def:scoring-rule}). 
Any such linear function is minimized at one of its boundaries, meaning that it is not uniquely minimized at $\bar p(y) = f$ unless $\bar p(y)$ is a unit vector; hence, all scoring rules violate CPR.

However, scoring rules do not necessarily violate the alignment axioms.
In fact, for every Bregman divergence, there is a scoring rule that gives the same difference in losses between any two predictions on every dataset. 
For example, this relationship holds between the Brier score and squared L2 error.
To state this fact more generally, we recall a classic result characterizing the set of scoring rules satisfying DP.

\begin{theorem} \citep[][Theorem 1.]{Gneiting2007}
    A scoring rule satisfies DP if and only if there exists a strictly convex function $B: \Delta(A) \to \R$ and subgradient $dB$ such that, for all $f \in \Delta(A)$ and $a \in A$,
    \[
        S(f, a) = -B(f) - dB(f)^T (e_a - f).
    \]
    Furthermore, every such scoring rule satisfies SP.
\end{theorem}

Now, suppose that $\Obj(f, y) = \nabla_{(B,dB)}(\bar p(y), f)$ is a Bregman divergence, 
and consider the alternative loss $\Obj'(f, y) = \Obj(f, y) + c(y)$, where $c(y)$ is an arbitrary function that depends only on the data. % (and not on the prediction).
This additive shift maintains the difference in losses between any two models on every dataset, and it is straightforward to show that it does not affect the status of any of the alignment axioms.
In particular, setting $c(y) = -B(\bar p(y))$ makes $\Obj'(f,y)$ a scoring rule.

What's more, these scoring rules are computationally easier to minimize than their corresponding DBBDs.
Scoring rules can be computed without explicitly calculating $\bar p(y)$, making them ideal for large datasets, as the loss can be evaluated without loading the entire dataset into memory at once.
Therefore, we do not recommend against the use of scoring rules for model training---it may often be a good idea!
We simply argue that researchers should use a corresponding DBBD when evaluating model performance.

\section{Conclusions}
Our goal in this paper was to identify suitable loss functions for evaluating behavioral models.
We took an axiomatic approach, developing axioms describing alignment and interpretability properties that such a loss function should satisfy.
We showed that almost all of the loss functions used in the field of behavioral game theory, including the entire class of scoring rules, violate at least one of these axioms.
However, it is indeed possible to construct loss functions that satisfy all of our axioms: we identified a large class---the diagonal bounded Bregman divergences---that does.  
Thus, we advocate that behavioral modelling work use one of these loss functions, with the squared L2 error as a natural incumbent. 

Although our motivation comes from behavioral game theory, 
recall that our arguments rely only on four characteristics of the field: 
the existence of a mapping from settings to finite, discrete distributions; 
the ability to obtain multiple observations for any setting; 
the goal of finding predictive models; 
and the need for these models to be interpretable.
Thus, our work provides guidance not only to behavioral game theorists, but to other researchers whose fields share these characteristics.
We are aware of examples in behavioral economics~\cite{Plonsky2019, Agrawal2020} and further afield in psychology~\cite{Busemeyer1993} and operations research~\cite{Hensher2000, Brenner2022}, 
and believe that there are yet more potential applications in political science and ecology.
We hope that our axiomatic view can help researchers across these disparate areas evaluate and interpret the performance of their models.

\longver{\subsection{Limitations and Future Work}}
\shortver{\paragraph{Limitations and Future Work.}}
All four of the characteristics played a role in our analysis: 
finite discrete distributions allowed us to formalize CPR; multiple observations motivated EDS and ZM; predictive models motivated DP and DPA; and interpretable models motivated DPA and SPA. 
This makes it clear that DP and DPA are not intended for descriptive modelling work, which focuses only on in-sample fit, and that DPA and SPA are unnecessary for evaluating high-capacity uninterpretable models such as deep neural nets, where propriety is sufficient.
The impact of our interpretability axioms is also limited, as they are not well motivated for modelling continuous distributions, such as energy consumption or climate variables, or in cases where only one sample can be observed, such as forecasting precipitation types.
It would be valuable to extend our results to these fields by developing suitable analogues of our axioms, lifting the need for discrete distributions or finding principled ways to aggregate similar observations.

Is it possible to make a theoretical argument for a \textit{single} best loss function?
If so, the path forward is to identify additional desirable axioms for loss functions in behavioral research.
For example, on ``rock-paper-scissors'' experiments, one might insist that loss functions be agnostic to the actions' identities,  ensuring that they do not treat ``rock'' differently from ``paper'' or ``scissors''.
Making compelling arguments for new axioms and understanding how they narrow down the space of permissible losses---indeed, whether any remain at all---is a valuable direction for future work.

\section*{Acknowledgements}
Thanks to Frederik Kunstner and Victor Sanches Portella for helpful discussions.
This work was funded by an NSERC CGS-D scholarship, 
an NSERC USRA award, 
an NSERC Discovery Grant, 
a DND/NSERC Discovery Grant Supplement, 
a CIFAR Canada AI Research Chair (Alberta Machine Intelligence Institute), 
awards from Facebook Research and Amazon Research, 
and DARPA award FA8750-19-2-0222, CFDA \#12.910 (Air Force Research Laboratory).

\bibliography{references}

\clearpage
\appendix

\appendixpage

We provide the following supplementary material in this appendix:
\begin{itemize}
    \item In \Cref{appendix:bregman_characterization}, we formally state and prove \Cref{thm:bregman_characterization_informal}.
    This theorem characterizes all ``well-behaved'' loss functions that satisfy SP, DP, and ZM, extending a result from \citetAppx{Abernethy2012}.
    \item In \Cref{appendix:dbbd_proof}, we prove \Cref{thm:dbbd}, showing that every DBBD satisfies all of the axioms. 
    \item In \Cref{appendix:sr_negative}, we prove \Cref{thm:scoring-rules-negative-result}, showing that no scoring rule can satisfy all of the axioms. 
    \item In \Cref{appendix:multiplegames}, we give a straightforward extension of the axioms to handle evaluations on multiple games. 
    \item In \Cref{appendix:real_data}, we empirically show that many axiom violations and interpretability problems discussed in the main text arise on real data.
\end{itemize}

\section{Formal Statement and Proof of Theorem~\ref{thm:bregman_characterization_informal}}
\label{appendix:bregman_characterization}

To formally state the first half of \Cref{thm:bregman_characterization_informal}, we need the following definitions. 

\begin{definition}
    Fix $n \in \mathbb{N}$, and let $L: \Delta(A) \times A^n \to \overline{\R}$ be a loss function. 
    Then $\Gamma: \Delta(A^n) \to \Delta(A)$ is a \emph{minimizer} of $L$ if for all $\mu \in \Delta(A^n), $\[\Gamma(\mu) \in \arg\min_{f \in \Delta(A)} \E_{y \sim \mu} L(f,y).\]
\end{definition}

When $\Obj$ satisfies DP, then $\Gamma(p^n) = p$. Note that \citetAppx{Abernethy2012} defines a loss function $\Obj$ to be \emph{proper for} a given property $\Gamma$ if $\Gamma$ minimizes $\Obj$; we instead begin with $\Obj$ and obtain $\Gamma$ for which $\Obj$ is proper by construction.

\begin{definition}
    Let $L: \Delta(A) \times A^n \to \overline{\R}$ be a loss function satisfying DP, and let $\Gamma$ be a minimizer of $L$. $L$ is \emph{$\Gamma$-differentiable} if for all $f \in \textrm{relint}(\Delta(A))$, and $\mu \in \Gamma^{-1}(f)$, the directional derivative \[\lim_{\epsilon \to 0} \frac{\E_{y \sim \mu}L(f + \epsilon v,y) - \E_{y \sim\mu} L(f, y)}{\epsilon}\] exists for all $v$ such that $f + \epsilon v \in \Delta(A)$ for sufficiently small $\epsilon$.
\end{definition}

We can now formally state \Cref{thm:bregman_characterization_informal}.
We split the statement into three parts.
First, holding $n$ fixed, we prove the only-if direction, showing that any loss function satisfying SP and DP must be a particular form of Bregman divergence (and that ZM further constrains this form).
Second, continuing to hold $n$ fixed, we prove the if direction, showing that all of these Bregman divergences indeed satisfy SP and DP (and, with the additional constraint, ZM).
Third, allowing $n$ to vary, we show that EDS constrains the Bregman divergences for each $n$ to be equivalent.

\begin{theorem}[\Cref{thm:bregman_characterization_informal}; fixed $n$, only-if direction]
    Fix $n \in \mathbb{N}$, and let $\Obj: \Delta(A) \times A^n \to \overline{\R}$ satisfy SP and DP. 
    Suppose that $\Obj$ is $\Gamma$-differentiable for some minimizer $\Gamma$ of $\Obj$. 
    Then there exists some closed and proper strictly convex function $B$ and subgradient $dB$, and some translation $c$ such that $\Obj$ is of the form \begin{kequation*}
        \Obj(f,y) = \nabla_{(B, dB)}(\bar p(y), f) + c(y),
    \end{kequation*}
    for all $f \in \textrm{relint}(\Delta(A)), y \in A^n$. 
    If $\Obj$ also satisfies ZM, then $c(y) = 0$ for all $y \in A^n$. 
\end{theorem}

We prove this theorem by adapting the proof of Theorem 11 of \citetAppx{Abernethy2012}.

\begin{proof}
    Taking $U = \Delta(A)$ and $\Omega = A^n$, we can apply Theorem 11
    of \citetAppx{Abernethy2012} to find that there exists some convex function $B$ and subgradient $dB$ as well as functions $\rho: A^n \to \Delta(A)$ and $c: A^n \to \R$ such that for all $f \in \textrm{relint}(\Delta(A)), y \in A^n$, \[\Obj(f,y) = \nabla_{(B,dB)}(\rho(y), f) + c(y).\]

    Moreover, $B$ is \emph{strictly} convex. This would follow immediately if reporting the true distribution uniquely minimized the expected loss over \emph{arbitrary} data distributions---allowing for arbitrary correlations between samples---but our DP axiom only requires this for i.i.d.\ data. 
    % when the data is an i.i.d. sample.
    However, it is possible to repair this problem with the following modification to their proof.
    Given a basis $\{b_i\}$ of $\Delta(A)$, they select a set of corresponding distributions $\mu_i \in \Gamma^{-1}(b_i)$; by the DP axiom, we know that $b_i^n \in \Gamma^{-1}(b_i)$, so we can pick $\mu_i = b_i^n$. 
    Then, for any linear combination $f = \sum_i \alpha_i b_i$, the distribution $\hat \mu[f] = \sum_i \alpha_i \mu_i$ samples i.i.d.\ from $f$. 
    Using this choice of distributions $\mu_i$, the remainder of their proof only requires that $\Obj$ be proper for distributions of $n$ i.i.d.\ observations. 
    In this case, our DP axiom gives a strict inequality $B(f) + dB(f)^T (f' - f) < B(f')$ for all $f \neq f'$; by Proposition D.6.1.3 of \citetAppx{Hiriart-Urruty2001}, this yields strict convexity of $B$.

    Since a Bregman divergence $\nabla_{(B,dB)}(p,q)$ of strictly convex $B$ is uniquely minimized by $p=q$, for any $y \in A^n$, $L(f,y)$ is uniquely minimized by $f = \rho(y)$; thus, SP constrains that $\rho(y) = \bar p(y)$ for all $y$.

    Finally, ZM implies that for all $y \in A^n$, \[L(\bar p(y), y) = \nabla_{(B,dB)}(\bar p(y), \bar p(y)) + c(y) = 0 + c(y) = 0,\] so $c(y) = 0$ for all $y$. Thus $L(f,y) = \nabla_{(B,dB)}(\bar p(y), f)$ for all $y \in A^n$, $f \in \textrm{relint}(\Delta(A))$.
\end{proof}

\begin{theorem}[\Cref{thm:bregman_characterization_informal}; fixed $n$, if direction] 
If $\Obj(f,y) = \nabla_{(B, dB)}(\bar p(y), f) +c(y)$ for some closed and proper strictly convex function $B: C \to \R$ such that $\Delta(A) \subseteq C$, subgradient $\smash{dB: C\times C \to \overline{\R}^d}$ of $B$, and translaton $c: A^n \to \R$, then $\Obj$ satisfies SP and DP. 
If $c(y) = 0$ for all $y$, then $\Obj$ also satisfies ZM.
\end{theorem}

All three of these axioms follow from a basic property of Bregman divergences of all strictly convex functions, which is that $\nabla_{(B, dB)}(p, q) \ge 0$, with equality if and only if $p = q$.

\begin{proof}
    \textbf{SP.} For any Bregman divergence with strictly convex $B$ and fixed $\bar p(y)$, $\nabla_{(B, dB)}(\bar p(y), f)$ is uniquely minimized by $f =\bar p(y)$. Since $y$ is fixed, the translation by $c(y)$ does not affect the minimizer.
    
    \textbf{DP.} This was shown by \citetAppx{Banerjee2005} for $f, p \in \textrm{relint}(\Delta(A))$; we show it generally below. Let $p \in \Delta(A)$; notice that $\E_{y \sim p^n} \bar p(y) = p$. Then, for all $f \in \Delta(A)$,\begin{align*}
        &\E_{y \sim p^n} \Obj(f,y)\\
        &= \E_{y \sim p^n}[\nabla_{(B,dB)}(\bar p(y), f) + c(y)]\\
        &=\E_{y \sim p^n} c(y) + \E_{y \sim p^n} B(\bar p(y))\\
        &\quad\quad - B(f) - \E_{y \sim p^n} (\bar p(y) - f)^T dB(f)\\
        &= \E_{y \sim p^n} c(y) + \E_{y \sim p^n} B(\bar p(y))  \\
        &\quad\quad- B(p) + B(p) - B(f) - (p - f)^T dB(f)\\
        &=\E_{y \sim p^n} c(y) + \E_{y \sim p^n} B(\bar p(y))- B(p) + \nabla_{(B,dB)}(p, f) .
    \end{align*}
    The first three terms do not depend on $f$, 
    and the final term is uniquely minimized by $f = p$.

    \textbf{ZM.} For any Bregman divergence, $\nabla_{(B,dB)}(\bar p(y), \bar p(y)) = 0$, so if $c(y) = 0$, \[L(\bar p(y), y) = \nabla_{(B,dB)}(\bar p(y), \bar p(y)) + 0 = 0.\]
\end{proof}

\begin{lemma}\label{lemma:equal_bd}
Let $p \in \Delta(A)$ be a distribution, 
and let $B_1, B_2: \Delta(A) \to \R$ be two closed and proper strictly convex functions with subgradients $dB_1$ and $dB_2$, respectively.
Then, $\nabla_{(B_1, dB_1)}(p, f) = \nabla_{(B_2, dB_2)}(p, f)$ for all $f \in \Delta(A)$ if and only if $B_1(f) - B_2(f) = c_1^T f + c_2$ for all $f \in \Delta(A)$ and some constants $c_1 \in \R^d, c_2 \in \R$.
In particular, this implies that $\nabla_{(B_1, dB_1)}(p', f) = \nabla_{(B_2, dB_2)}(p', f)$ for all $p', f \in \Delta(A)$.
\end{lemma}

\begin{proof}
    The if direction is straightforward: we have
    \begin{align*}
        &\nabla_{(B_1, dB_1)}(p, f) \\
        = &B_1(p) - B_1(f) - dB_1(f)^T (p - f) \\
        = &\left(B_2(p) + c_1^T p + c_2\right) - \left(B_2(f) + c_1^T f + c_2\right) \\
          &- \left(dB_2(f)^T + c_1\right)^T (p - f) \\
        = &B_2(p) - B_2(f) - dB_2(f)^T (p - f) \\
          &+ \underbrace{c_1^T p + c_2 - c_1^T f - c_2 - c_1 (p - f)}_{0} \\
        = &\nabla_{(B_2, db_2)}(p, f). 
    \end{align*}

    For the only-if direction, let $D(f) = B_1(f) - B_2(f)$ be the difference between the two functions, 
    and let $dD(f) = dB_1(f) - dB_2(f)$ be the difference between their subgradients.
    We aim to show that $D$ is an affine function $D(f) = c_1 f + c_2$ for some constants $c_1, c_2$.

    Expanding the definition of Bregman divergences gives and rearranging terms gives
    \begin{align*}
          &B_1(p) - B_1(f) - dB_1(f)^T (p - f) \\
        = &B_2(p) - B_2(f) - dB_2(f)^T (p - f);
    \end{align*}
    rearranging terms, we have
    \[
        dD(f)^T (f - p) - D(f) = -D(p).
    \]
    This is a partial differential equation for $D(f)$.
    To show that $D(f)$ is affine, we will solve this equation along an arbitrary ray $f = p + tv$, 
    where $t \in \R$ and $v \in \R^d$, with $\sum_i v_i = 0$.
    Let $R(t) = D(p + tv)$.
    Then, $R'(t) = v^T dD(p + tv)$, and the differential equation becomes
    \[
        tR'(t) - R(t) = -D(p).
    \]
    This is now a linear ordinary differential equation with the solution
    \[
        R(t) = D(p) + Ct
    \]
    for some $C \in \R$.
    Thus, $D(f)$ is affine along any arbitrary ray.
    This implies that it is an affine function, completing the proof. 
\end{proof}

\begin{theorem}[\Cref{thm:bregman_characterization_informal}, varying $n$]
Let $\mathcal{B}$ be a family of closed and proper strictly convex functions with subgradients $d\mathcal{B}$, 
and let $\Obj(f, y) = \nabla_{(\mathcal{B}_n, d\mathcal{B}_n)}(\bar p(y), f)$.
Then, $\Obj(f, y)$ satisfies EDS if and only if there is some convex function $B$ and subgradient $dB$ such that $\nabla_{(\mathcal{B}_n, d\mathcal{B}_n)} = \nabla_{(B, dB)}$ for all $n$.
\end{theorem}

\begin{proof}
The if direction is immediate. 
For the only if direction, we will show that the theorem statement holds with $B = \mathcal{B}_1$ and $dB = d\mathcal{B}_1$.
Let $y^{(1)}$ be a dataset consisting of a single action, and let $y^{(n)}$ be a dataset consisting of $n$ copies of the same action.
Because $\bar p(y^{(1)}) = \bar p(y^{(n)})$, EDS requires that, for all $f$,
\begin{align*}
    \Obj(f, y^{(1)}) &= \Obj(f, y^{(n)}) \\
    \implies 
    \nabla_{(\mathcal{B}_1, d\mathcal{B}_1)}(f, \bar p(y^{(1)}))
    &= \nabla_{(\mathcal{B}_n, d\mathcal{B}_n)}(f, \bar p(y^{(1)})).
\end{align*}
Applying \Cref{lemma:equal_bd} completes the proof.
\end{proof}

\section{Proof of Theorem~\ref{thm:dbbd}}\label{appendix:dbbd_proof}
We will now prove that every diagonal bounded Bregman divergences satisfies all of our axioms.

First, recall the definition of a DBBD.

\dbbddefn*

Note that while $b'$ is not defined at 0, since $b$ is continuously differentiable and convex, $b'$ is monotonic and we can define $b'$ at the endpoints as the continuous extension, which we constrain to be finite. For the remainder of this argument, we simplify $\nabla_{(B_b, \nabla B_b)}$ to $\nabla_{B_b}$.

We are now ready to prove \Cref{thm:dbbd}.

\dbbd*

\begin{prf} 
    Let $\Obj$ be a DBBD, where $\Obj(f,y) = \nabla_{B_b}(\bar p(y), f)$ for some $b$. 
    We establish that every such loss function satisfies each of our axioms in turn.
   
    \textbf{SPA.} We use the convexity of $b$ to prove that $\Obj$ satisfies SPA.
    Here, note that $\Obj$ is bounded on $[0,1]$ due to continuity on $[0,1]$ and the bounded first derivative of $b$.
    
    Let $f, g \in \Delta(A)$, $y \in A^n$. Denote $\bar p = \bar p(y)$, and suppose $f \succ_{\bar p} g$. 
    
    As $b$ is convex and differentiable on $[0,1]$, $b'$ is increasing. 
    
    Let $1 \leq a \leq d$, and suppose that $\bar p_a \leq f_a \leq g_a$. Then $\bar p_a - f_a \leq 0$, and $b'(f_a) \leq b'(g_a)$.
    
    Now suppose that $\bar p_a \geq f_a \geq g_a$. Then $\bar p_a - f_a \geq 0$, and $b'(f_a) \geq b'(g_a)$.
    
    In either case, $(\bar p_a - f_a) b'(f_a) \geq (\bar p_a - f_a) b'(g_a)$. Then \begin{align*}
        &\Obj(g,y) - \Obj(f,y)\\
        &= \nabla_{B_b} (\bar p, g) - \nabla_{B_b} (\bar p, f)\\
        &= B_b(\bar p) - B_b(g) - (\bar p - g)^T \nabla B_b(g)\\
        &\quad - \left(B_b(\bar p) - B_b(f) - (\bar p - f)^T \nabla B_b (f)\right)\\
        &= \sum_{a=1}^d b(f_a) - b(g_a) + (\bar p_a - f_a)b'(f_a) - (\bar p_a - g_a) b'(g_a)\\
        &\geq \sum_{a=1}^d b(f_a) - b(g_a) + (\bar p_a - f_a) b'(g_a) - (\bar p_a - g_a) b'(g_a)\\
        &= \sum_{a=1}^d \nabla_b(g_a, f_a)
    \end{align*}
    For all $a$, $\nabla_b(g_a, f_a) \geq 0$. For some $\tilde a$, $g_{\tilde a} \neq f_{\tilde a}$, so $\nabla_b(g_{\tilde a}, f_{\tilde a}) > 0$, and $\Obj(g,y) - \Obj(f,y) > 0$.
    
    \textbf{DPA.} Let $p\in \Delta(A)$ and notice that $\E_{y \sim p^n} \bar p(y) = p$. 
    
    For any $f, g \in \Delta(A)$, $\nabla_{B_b}(\bar p(y), f) - \nabla_{B_b}(\bar p(y), g)$ is linear in $\bar p(y)$ as the term $B_b(\bar p(y))$ in each cancels. Thus \begin{align*}
        &\E_{y\sim p^n} \Obj(f,y) - \E_{y \sim p^n} \Obj(f,y)\\
        &= \E_{y \sim p^n} \left[\nabla_{B_b}(\bar p(y), f) - \nabla_{B_b}(\bar p(y), g)\right]\\
        &= \nabla_{B_b}(p, f) - \nabla_{B_b}(p, g).
    \end{align*} Thus if $f \succ_p g$, we can repeat the argument that $\Obj$ satisfies SPA with $\bar p$ replaced by $p$, and conclude that $\E_{y \sim p^n} \Obj(f,y) - \E_{y \sim p^n} \Obj(g,y) > 0$.

    \textbf{Ex.} This follows from \Cref{thm:bregman_characterization_informal}.

    \textbf{CPR.} This argument follows a similar structure to the proof of SPA.
    
    Let $n \in \mathbb{N}$, and $y, y' \in  A^n$. Let $f \in \Delta(A)$. Denote $\bar p := \bar p(y)$, $\bar p' := \bar p(y')$, with $\bar p \succ_f \bar p'$.
    
    As $b$ is convex and differentiable on $[0,1]$, $b'$ is increasing.
    
    Let $1 \leq a \leq d$, and suppose that $\bar p_a' \leq \bar p_a \leq f_a$. Then we know that $\bar p_a' - \bar p_a < 0$, and $b'(f_a) \geq b'(\bar p_a)$. 
    
    On the other hand, if $\bar p_a' \geq \bar p_a \geq f_a$, then $\bar p_a' - \bar p_a > 0$, but $b'(f_a) \leq b'(\bar p_a)$. 
    
    Thus in both cases, $(\bar p_a' - \bar p_a) b'(f_a) \leq (\bar p_a' - \bar p_a) b'(\bar p_a)$. Then \begin{align*}
    &\Obj(f,y') - \Obj(f,y)\\
    &= \nabla_{B_b}(\bar p', f) - \nabla_{B_b}(\bar p, f)\\
    &= B_b(\bar p') - B_b(f) - (\bar p' - f)^T \nabla B_b(f)\\
    &\quad- \left(B_b(\bar p) - B_b(f) - (\bar p - f)^T \nabla B_b(f)\right)\\
    &= B_b(\bar p') - B_b(\bar p) -(\bar p' - \bar p)^T\nabla B_b(f)\\
    &= \sum_{a=1}^d b(\bar p_a') - b(\bar p_a) - (\bar p_a' - \bar p_a) b'(f_a)\\
    &\geq \sum_{a=1}^d b(\bar p_a') - b(\bar p_a) - (\bar p_a' - \bar p_a) b'(\bar p_a)\\
    &= \sum_{a=1}^d \nabla_d(\bar p_a', \bar p_a).
    \end{align*}
    
    For all $a$, $\nabla_d(\bar p_a', \bar p_a) \geq 0$, and for some $\tilde a$, $\bar p_{\tilde a} \neq \bar p_{\tilde a}'$, so $\nabla_d(\bar p_{\tilde a}', \bar p_{\tilde a}) > 0$. Thus $\Obj(f,y') - \Obj(f,y) > 0.$
    
    \textbf{ZM.} This also follows from \Cref{thm:bregman_characterization_informal}.
    \end{prf}

\section{Proof of \Cref{thm:scoring-rules-negative-result}}\label{appendix:sr_negative}

For convenience, we restate the proposition here.

\srnegative*

\begin{prf}
    We begin by showing that a scoring rule cannot satisfy both SPA and ZM.
    Suppose $S$ is such a scoring rule; we will derive a contradiction.
    First, for all $a \in A$, if $f \neq e_a$, SPA requires that $S(f, a) = \Obj_S(f, \{a\}) > \Obj_S(e_a, \{a\}) = S(e_a, a)$.
    Second, for all $a \in A$, ZM requires that $S(e_a, a) = \Obj_S(e_a, \{a\}) = 0$.
    Then, for any dataset $y$ such that $\bar p(y)$ has support on at least 2 elements, we have 
    \begin{align*}
        0
        &= \Obj_S(\bar p(y), y)\\
        &= \sum_{a \in A} \bar p(y)_a S(\bar p(y), a)\\
        &> \sum_{a \in A} \bar p(y)_a S(e_a, a)\\
        &= 0,
    \end{align*}
    a contradiction.

    Now, we show that an arbitrary scoring rule does not satisfy CPR. Let $S$ be a scoring rule. Suppose $y$ is a set of observations in which not all players select the same action. Define $a^*$ to be an action in $A$ with the best score, i.e., \[a^* \in \arg\min_{a\in A}S(\bar p(y), a).\] 
    
    Let $y'$ be another set of observations of the same size where all players play $a^*$. Then \begin{align*}
        L_S(\bar p(y), y) &= \sum_{a \in A} \bar p(y)_a S(\bar p(y), a)\\
        &\geq S(\bar p(y), a^*)\sum_{a \in A} \bar p(y)_a \\
        &= S(\bar p(y), a^*)\\
        &= L_S(\bar p(y), y'),
    \end{align*} but $\bar p(y) \succ_{\bar p(y)} \bar p(y')$.
\end{prf}

\section{Extension to Multiple Settings}\label{appendix:multiplegames}

In the main text, we made the simplifying assumption that a behavioral model will be evaluated on a single setting.
In general, though, models may predict distributions over outcomes for every instance in a set of strategic settings; the ideal model should perform well on every setting.
Then, given data for a subset of the possible settings, a modeller's goal is to compute an \emph{aggregate} loss to gauge the overall performance of a model.
In this section, we discuss how to extend our axioms to handle data from multiple settings.

With multiple settings, there are broadly two types of data.
The first is \emph{panel} data, where each participant is observed acting in each settings, allowing the modeller to observe correlations between actions in different settings.
Evaluating models on panel data is conceptually simple.
One can simply think of modelling the \emph{joint} setting, in which each observation consists of a tuple of actions (one for each setting) and a model predicts a joint distribution over these tuples.
The exponential size of this joint observation space may lead to practical issues; we leave these considerations for future work.
Instead, in the remainder of this section, we focus on \emph{pooled} data, where it is not possible to link observations across settings; this includes datasets where participant identities are removed, or where a distinct set of participants are observed in each setting.

We first extend our notation. There is a set $S$ of settings; each setting $s \in S$ has a finite set of outcomes $A(s)$ and a mapping to a distribution $p(s)$ over $A(s)$. The researcher wishes to evaluate a model $f$, which is a function producing a prediction of $p(s)$ for each $s \in S$. The researcher has data $y^i$ of observations in setting $s_i$ for $m$ settings. The loss $\Agg$ then takes a collection of $m$ inputs of the form $(s_i, y^i, f(s_i))$ and produces a non-negative real number describing the model's overall performance. The loss should handle an arbitrary number settings.

We begin with an extension of DP, which states that predicting the true distribution for \emph{every} setting produces a strictly optimal expected aggregate loss.
\begin{axiom}
    For all $m \geq 1, \{s_i\}_{i=1}^m$, and any model $f$ and dataset sizes $\{n^i\}_{i=1}^m$, if \[f(s_i) \neq p(s_i)\] for some $i$, then \begin{align*}
        &\E_{\substack{y^i \sim p(s_i)^{n_i},\\ 1\leq i \leq m}} \Agg(\{(s_i,y^i, f(s_i))\}_{i=1}^m)\\
        &> \E_{\substack{y^i \sim p(s_i)^{n_i}, \\ 1 \leq i \leq m}} \Agg(\{(s_i,y^i, p(s_i))\}_{i=1}^m).
    \end{align*}
\end{axiom}
To define DPA in this setting, we need an additional definition.
\begin{definition}[Weak Pareto improvement]
    Let $p, q, r \in \Delta(A)$. We say that $q$ is a \emph{weak Pareto improvement} over $p$ with respect to $r$, denoted by $q \succsim_r p$, if for all $a \in A$, f either $p_a \leq q_a \leq r_a$ or $p_a \geq q_a \geq r_a$.
\end{definition} Note that this is the same definition as for Pareto improvement except that we do not require a strict improvement in any dimension.
The extension of DPA reflects the principle that if the model is to be evaluated on a given \emph{collection} of settings, the aggregate loss function should still be incentive-compatible.  We thus require that given an unambiguously better model -- i.e., weak Pareto improvements in every setting and a (``strict'') Pareto improvement in at least one setting -- the expected aggregate loss decreases.
\begin{axiom}
    For all $m \geq 1, \{s_i\}_{i=1}^m$, any two models $f$ and $g$, and any dataset sizes $\{n^i\}_{i=1}^m$, if \[f(s_i) \succsim_{p(s_i)} g(s_i)\] for all $i$ and \[f(s_j) \succ_{p(s_j)} g(s_j)\] for some $j$, then \begin{align*}
        &\E_{\substack{y^i \sim p(s_i)^{n_i},\\ 1\leq i \leq m}}\Agg(\{(s_i,y^i, f(s_i))\}_{i=1}^m)\\
        &<\E_{\substack{y^i \sim p(s_i)^{n_i},\\ 1\leq i\leq m}} \Agg(\{(s_i,y^i, g(s_i))\}_{i=1}^m).
    \end{align*}
\end{axiom}
SP now requires that at evaluation time, a perfect prediction on \emph{every} setting achieves a strictly optimal aggregate loss.
\begin{axiom}
    For all $m \geq 1$, $\{s_i\}_{i=1}^m$, datasets $y^i$ in each $s_i$, and any model $f$, if \[f(s_i) \neq \bar p(y^i),\] for some $i$, then \[\Agg(\{(s_i, y^i, f(s_i))\}_{i=1}^m) > \Agg(\{(s_i, y^i, \bar p(y^i))\}_{i=1}^m).\]
\end{axiom}
Similarly, SPA now states that if the model is being evaluated on a \emph{collection} of settings, a model which is an unambiguously better match to the realized data should get a lower 
aggregate loss.
\begin{axiom} For all $m \geq 1$, settings $\{s_i\}_{i=1}^m$, datasets $y^i$ in each $s_i$, and models $f$ and $g$, if \[f(s_i) \succsim_{\bar p(y^i)} g(s_i)\] for all $i$ and \[f(s_j) \succ_{\bar p(y^j)} g(s_j)\] for some $j$, then \[\Agg(\{(s_i, y^i, f(s_i))\}_{i=1}^m) < \Agg(\{(s_i, y^i, g(s_i))\}_{i=1}^m).\] 
\end{axiom}
Next, if \emph{each} setting's dataset is rearranged, the loss remains the same.
\begin{axiom}
    For all $m \geq 1$, settings $\{s_i\}_{i=1}^m$, datasets $y^i$ in each $s_i$, models $f$, and permutations $\pi_i \in \Pi(\{1, ..., n(y^i)\})$ for $i \in [1,..., m]$, then \[\Agg(\{(s_i, y^i, f(s_i))\}_{i=1}^m) = \Agg(\{(s_i, \pi_i(y^i), f(s_i))\}_{i=1}^m).\]
\end{axiom}
CPR extends to require that if a \emph{collection} of datasets matches a model more closely than another collection of datasets, then the loss of the model on the first collection should be lower. 
\begin{axiom} For all $m \geq 1$, settings $\{s_i\}_{i=1}^m$, datasets $y^i$ and $z_i$ of the same size $n_i$ in each $s_i$, and models $f$,
if \[\bar p(y^i) \succsim_{f(s_i)} \bar p(z^i)\] for all $i$ and \[\bar p(y^j) \succ_{f(s_j)} \bar p(z^j)\] for some $j$, then $$\Agg(\{(s_i, y^i, f(s_i))\}_{i=1}^m) < \Agg(\{(s_i,z^i, f(s_i))\}_{i=1}^m).$$
\end{axiom}
Finally, ZM becomes a requirement that a model perfectly matching the realized data distribution for \emph{every} setting yields a loss of zero. 
\begin{axiom}
For all $m \geq 1$, $\{s_i\}_{i=1}^m$, and datasets $y^i$, \[\Agg(\{(s_i, y^i, \bar p(y^i))\}_{i=1}^m) = 0.\]
\end{axiom}

In practice, it is typical for modellers to compute the loss of the model's prediction on the data for each of the settings, and then take the (possibly weighted) average of the losses.
This is a sound choice: it is easy to show that if a loss function $\Obj$ satisfies the single-setting axioms, then the aggregate loss $\Agg$ that combines the setting-wise losses by taking any such weighted average will satisfy each of the axioms, as long as the weights on each setting are  strictly positive.
An interesting avenue for future work is to develop axioms further restricting unreasonable methods of aggregating losses.

\section{Examples on Real Data}
\label{appendix:real_data}
In this section, we demonstrate that the choice of loss function is of practical importance on real behavioral game theory data, 
showing that different losses imply different preferences over models 
and that several of the axiom violations shown in Section~\ref{sec:revisitexisting} appear in practice.
To do so, we use experimental data on two ``traveler's dilemma'' games from \citet{Goeree2001} to evaluate three models from \citet{Wright2017}.
In particular, we test Nash equilibria (which is well-defined on these games, as both have a single pure Nash equilibrium); level-k (with 9.3\% level-0 players playing the uniform random strategy, 53.8\% level-1 players, and 36.9\% level-2 players); and quantal cognitive hierarchy (with a Poisson(1.03) distribution of levels and precision 0.209).
We also compare to the loss of a perfect prediction $\bar p(y)$, which exactly matches the empirical distribution.

The results are shown in Figure~\ref{tab:real_data}.
There are several things to note about these results.
First, on the ``high penalty'' game (Figure~\ref{tab:real_data}a), error rate gives Nash a lower loss than level-k or QCH; cross-entropy, KL divergence, and NLL give Nash the highest loss; and MAE, Brier, and squared L2 rank Nash in the middle, showing that the various` losses imply different preferences over the models.
Second, both Nash and level-k get a lower error rate than a perfect prediction, violating SP and SPA.
Third, even the perfect prediction gets a non-zero error rate, cross-entropy, Brier score, and NLL, violating ZM in each case.
These highlight the potential issues in interpreting the losses: for instance, without seeing the loss of a perfect prediction, one might guess from the Brier scores that all three models have much room to improve.
NLL is especially difficult in this regard, as the irreducible error is proportional to the number of observations; we had to plot it on a separate y-axis to avoid distorting the plot.

Lastly, comparing losses on the two games, almost every combination of loss and prediction shows a higher loss on the ``high penalty'' than on the ``low penalty'' game. 
This makes sense, as all three models put substantial mass on the left-most action in the game, which is played by 76\% of participants in the high penalty game but only 8\% in the low penalty game.
The only exception is Nash, which receives an infinite loss on both games in terms of NLL, cross-entropy, and KL divergence.
Technically, this is not a violation of CPR, as both games have several actions that are played once, but never played in the other.
Still, we think that these cases are notable, as they demonstrate the broader problem of NLL, cross-entropy, and KL producing infinite losses, and we believe it is nonetheless difficult to argue that Nash is an equally good prediction in the two games.

\begin{figure*}[t]
\begin{center}
\begin{tabular}{ll}
    \includegraphics{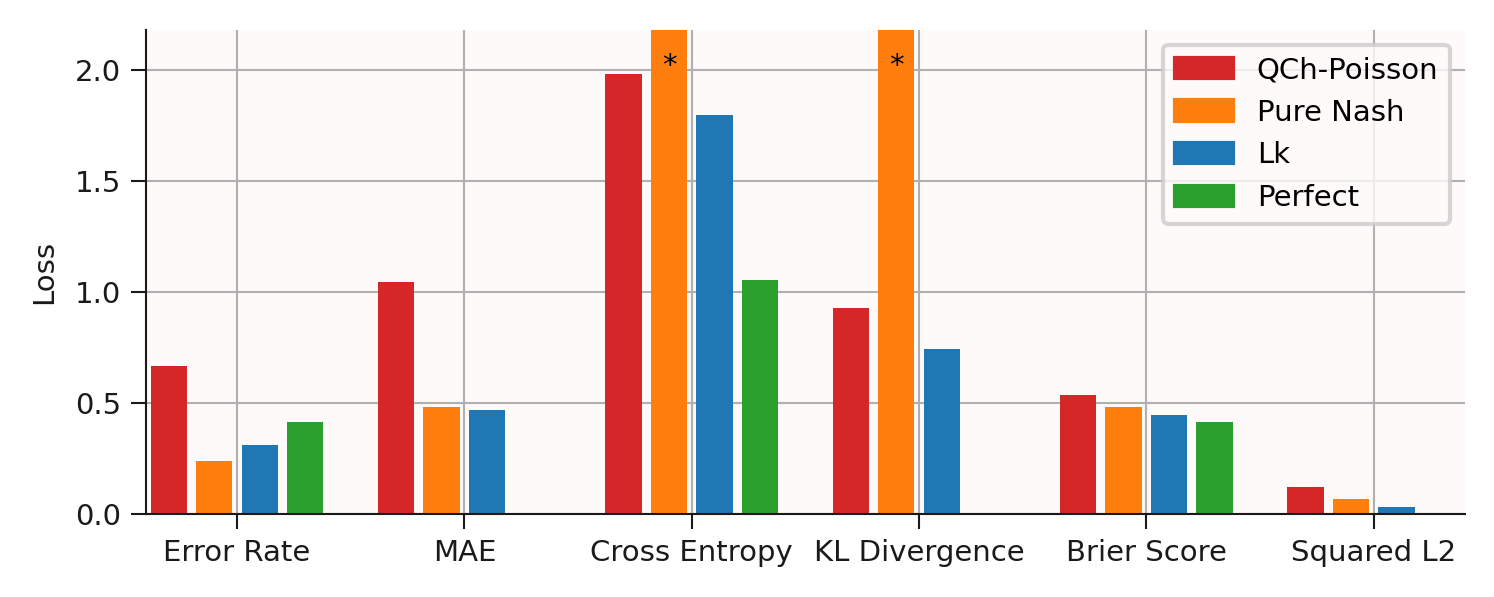} &
    \includegraphics{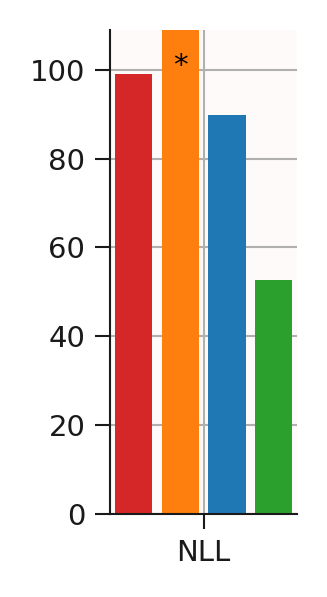} \\
    \multicolumn{2}{c}{(a) Traveler's Dilemma (High Penalty)} \\ 
    \\
    \includegraphics{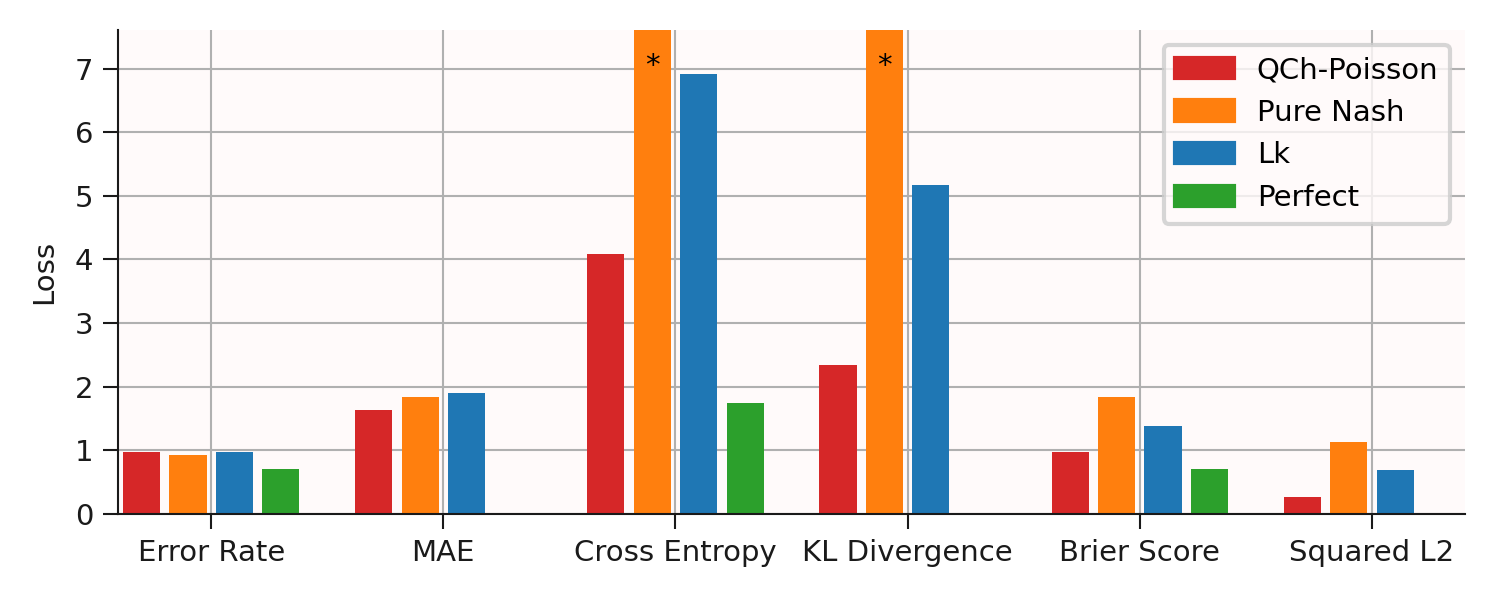} &
    \includegraphics{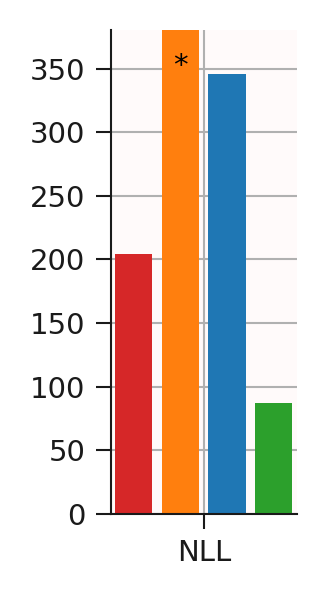} \\
    \multicolumn{2}{c}{(b) Traveler's Dilemma (Low Penalty)} \\ 
\end{tabular}
\caption{The losses of four predictions on two traveler's dilemma games~\cite{Goeree2001}.}
\label{tab:real_data}
\end{center}
\end{figure*}

\bibliographystyleAppx{plainnat}
\bibliographyAppx{references}

\end{document}